%% file: root.tex
\documentclass[letterpaper, 10 pt, conference]{ieeeconf}  % Comment this line out if you need a4paper

\IEEEoverridecommandlockouts                              % This command is only needed if 
\usepackage{graphicx} % for \graphicspath
\usepackage{caption}
\usepackage{subcaption}

\usepackage{xcolor,comment}
\usepackage{algorithm}
\usepackage[noend]{algpseudocode} % algorithmic, \State{}, etc
\usepackage{amssymb,amsmath}
\newtheorem{theorem}{Theorem}
\newtheorem{lemma}{Lemma}
\newtheorem{definition}{Definition}
\usepackage{url}
\usepackage[utf8]{inputenc}

\usepackage{ifthen}
\newboolean{shortver}
\setboolean{shortver}{true}% for short version 
\title{\LARGE \bf
MS*: A New Exact Algorithm for Multi-agent Simultaneous Multi-goal Sequencing and Path Finding
}

\author{Zhongqiang Ren$^{1}$, Sivakumar Rathinam$^{2}$ and Howie Choset$^{1}$% <-this % stops a space
	\thanks{$^{1}$ Zhongqiang Ren and Howie Choset are with the Robotics Institute and the Department of Mechanical Engineering at Carnegie Mellon University, 5000 Forbes Ave., Pittsburgh, PA 15213, USA. Emails: \{zhongqir, choset\}@andrew.cmu.edu
	}%
	\thanks{$^{2}$Sivakumar Rathinam is with the Department of Mechanical Engineering, Texas A\&M University,
		College Station, TX 77843-3123.
		Email: srathinam@tamu.edu
	}
}

\begin{document}

\maketitle
\thispagestyle{plain}
\pagestyle{plain}
\pagenumbering{arabic}
%\pagenumbering{num_style}

%%%%%%%%%%%%%%%%%%%%%%%%%%%%%%%%%%%%%%%%%%%%%%%%%%%%%%%%%%%%%%%%%%%%%%%%%%%%%%%%
\begin{abstract}
	\input{abstract}
\end{abstract}

%%%%%%%%%%%%%%%%%%%%%%%%%%%%%%%%%%%%%%%%%%%%%%%%%%%%%%%%%%%%%%%%%%%%%%%%%%%%%%%%

\graphicspath{{figures/}}

\vspace{-1mm}
\section{Introduction}
\vspace{-1mm}
\input{intro}

\vspace{-1mm}
\section{Prior Work}\label{sec:prior_work}
\vspace{-1mm}

\input{prior_work}

\vspace{-1mm}
\section{Problem Description}\label{sec:problem}
\vspace{-1mm}
\input{problem_def}

\vspace{-1mm}
\section{MS* Algorithm}\label{sec:method}
\vspace{-1mm}
\input{method}

\vspace{-1mm}
\section{Analysis}\label{sec:analysis}
\vspace{-1mm}
\input{analysis}

\vspace{-1mm}
\section{Numerical Results}\label{sec:result}
\vspace{-1mm}
\input{result}

\vspace{-1mm}
\section{Conclusion}\label{sec:conclude}
\vspace{-1mm}
\input{conclude}

%\section*{APPENDIX}

%\section*{ACKNOWLEDGMENT}

\bibliographystyle{plain}
\bibliography{references}

\end{document}

%% file: abstract.tex
In multi-agent applications such as surveillance and logistics, fleets of mobile agents are often expected to coordinate and safely visit a large number of goal locations as efficiently as possible. The multi-agent planning problem in these applications involves allocating and sequencing goals for each agent while simultaneously producing conflict-free paths for the agents. In this article, we introduce a new algorithm called MS* which computes an optimal solution for this multi-agent problem by fusing and advancing state of the art solvers for multi-agent path finding (MAPF) and multiple travelling salesman problem (mTSP). MS* leverages our prior subdimensional expansion approach for MAPF and embeds the mTSP solvers to optimally allocate and sequence goals for agents. Numerical results show that our new algorithm can solve the multi-agent problem with 20 agents and 50 goals in a minute of CPU time on a standard laptop.

%% file: intro.tex
\graphicspath{{figures/}}

Multi-agent applications such as information acquisition in surveillance or package delivery in logistic services require agents to safely visit a large number of goal locations as efficiently as possible.
A fundamental problem in these applications, referred to as the Multi-agent simultaneous multi-goal sequencing and path finding (MSMP), aims to allocate and sequence goals for each agent and plan conflict-free paths for the agents while minimizing the sum of the travel costs.
MSMP generalizes the single-agent path finding problem in~\cite{chour2021s} to multiple agents.
MSMP also extends the multi-agent path finding (MAPF) problem~\cite{ma2016optimal, stern2019multi} found in the robotics community as well as the well-known multiple traveling salesman problem (mTSP)~\cite{OberlineIEEEMagazine2010} addressed in the optimization community.
Specifically, MAPF generally assumes the number of goals is equal to the number of agents and are mainly about finding conflict-free paths for the agents.
On the other hand, the mTSP is a combinatorial problem that only addresses the allocation and sequencing of goals for agents without considering the conflicts between the paths of agents.
The MSMP considers both path finding and goal sequencing concurrently (Fig.~\ref{fig:test_case}).

\begin{figure}%[htbp]
  \centering
  \includegraphics[width=0.4\linewidth]{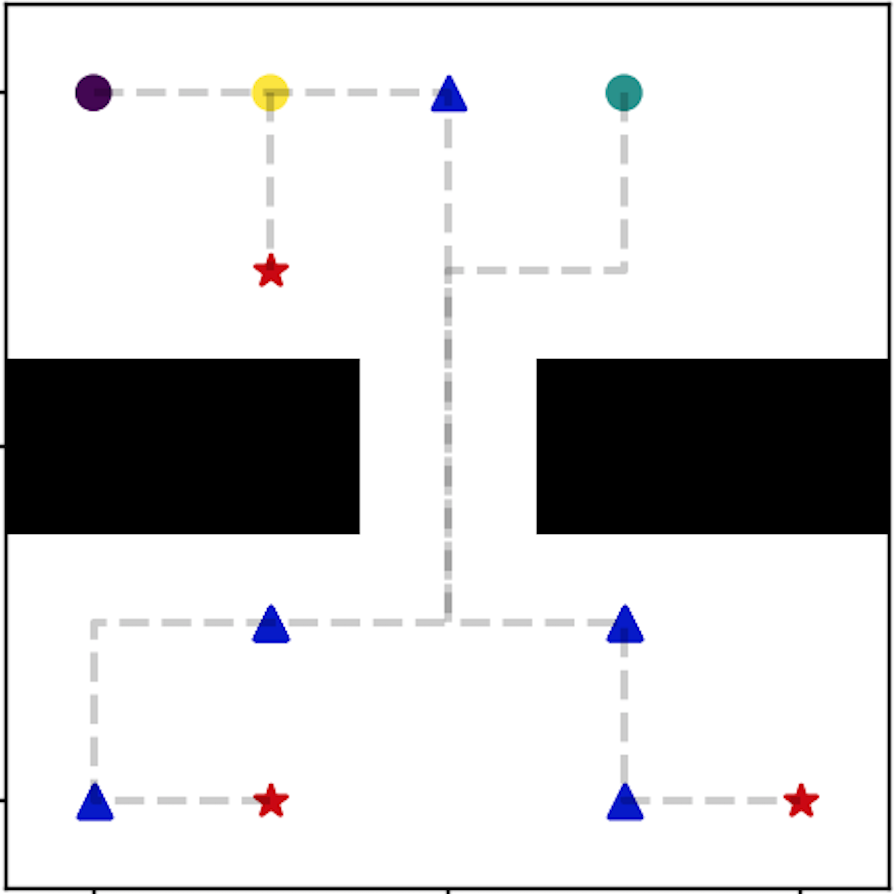}
  \vspace{-1mm}
	\caption{
	An illustration of the challenges in MSMP. The stars, triangles and circles represent the starts, goals and destinations respectively. Each goal must be visited in the middle of some agent's path while the planned paths of agents must also be conflict-free.
	}
	\vspace{-6mm}
	\label{fig:test_case}
\end{figure}

While several variants of the MAPF and mTSP have been addressed in the literature in isolation, we are not aware of any exact algorithm that finds an optimal solution for the MSMP problem.
In this work, we take our first step to address this gap and propose a new algorithm named as MS*\footnote{Multi Steiner* (MS*): M refers to multi-agents and S refers to a Steiner traveling salesman path or a tour for each agent.} and a variant named MS*-c. We show that MS*-c can solve the MSMP problem to optimality.
In MS* and MS*-c, we create a new state space and handle conflicts via subdimensional expansion \cite{wagner2015subdimensional} while simultaneously allocating and sequencing goals for agents via state of the art mTSP solvers.
Specifically, the subdimensional expansion method dynamically modifies the dimension of the new search space based on agent-agent conflicts and defers planning in the joint space until necessary.
Concurrently, the complexity in goal allocation and sequencing is addressed by embedding the mTSP solvers in the form of (1) heuristics that underestimate the cost-to-go from any state, and (2) individual optimal policies that constructs the low dimensional search space for subdimensional expansion.
To further reduce the computational complexity, we also define new dominance rules in the proposed state space to prune any partial solution that can not to lead to an optimal solution. Numerically, we perform simulations on instances from~\cite{stern2019multi} with at most 20 agents and 50 goals to verify the performance of the proposed approach.

The rest of the article is organized as follows. Sec.~\ref{sec:prior_work} provides a literature review related to MSMP. We formally define the MSMP problem in Sec.~\ref{sec:problem}.
Next, we describe the MS* algorithm in Sec.~\ref{sec:method}, and discuss its completeness and optimality properties in Sec.~\ref{sec:analysis}. Numerical results are presented in Sec.~\ref{sec:result} with conclusions in Sec.~\ref{sec:conclude}.

%% file: prior_work.tex
\noindent\textbf{Multi-agent path finding},
%\subsection{Multi-agent Path Finding}
%Multi-agent path finding (MAPF),
as its name suggests, computes an ensemble of collision-free paths that navigate a set of agents from their respective starts to destinations.
Finding optimal paths for MAPF is NP-hard \cite{yu2013structure_nphard} as the problem complexity increases exponentially with respect to the number of agents. MAPF methods tend to fall on a spectrum from decentralized to centralized, trading off completeness and optimality for scalability. Conflict-based search \cite{sharon2015conflict} and subdimensional expansion \cite{wagner2015subdimensional} are widely known approaches for MAPF that lie in the middle of the spectrum.
They are complete and guaranteed to provide optimal solutions; however, they typically handle scenarios where the number of goals or tasks is equal to the number of agents.

\ifthenelse{\boolean{shortver}}{% 
}{% 
Conflict-based search (CBS) \cite{sharon2015conflict} and its variants \cite{sharon2012meta, boyarski2015icbs,barer2014suboptimal} employs a two-level search, where conflicts along agent paths are detected in the high level search and constrained single agent path planning is leveraged to resolve conflicts in the lower level search. CBS has been extended to several planning problems as discussed in \cite{andreychuk2019multi,cohen2019optimal,honig2018conflict}. 
}

Subdimensional expansion \cite{wagner2015subdimensional}, our prior method, bypasses the curse of dimensionality by dynamically modifying (while attempting to bound) the dimension of the search space based on agent-agent conflicts.
This approach applies to many algorithms, and therefore inherits completeness and optimality, if the underlying algorithm already has these features, while scaling well on average in less crowded environments.
\ifthenelse{\boolean{shortver}}{% 
	Applying subdimensional expansion to A* results in M*.
}{%
While general to probablistic roadmaps (PRM) and rapid exploring random trees (RRT) \cite{wagner2012probabilistic}, much of prior work thoroughly explored the application of subdimensional expansion algorithm onto graph search approaches like A* resulting in M* \cite{wagner2015subdimensional}.
}
M* begins by computing an individual optimal policy for every agent and lets them follow the policies towards to their destinations, ignoring any other agents.
When a conflict is found, the subset of agents in conflict are coupled together and a corresponding new search space with increased dimensions is formed where a new search is conducted in order to resolve the conflict locally.

\vspace{1mm}
\noindent\textbf{Multi-goal sequencing} is related to
%\subsection{Multi-goal Sequencing}
the traveling salesman problem (TSP), which aims to find an optimal tour in a given graph is one of the most well known NP-Hard problems in the literature \cite{Applegate:2007}. 
Several methods have been developed \cite{Applegate:2007} for the TSP ranging from exact techniques (branch and bound, branch and price) to heuristics and approximation algorithms which trade off optimality for computational efficiency. 
There are numerous generalizations of the TSP which include multiple agents with motion, fuel, capacity and other resource constraints \cite{Riddhi_2010,sundar2016generalized}. 
Multiple TSP (mTSP) is much harder than solving a single TSP since goal locations must be partitioned and allocated to each agent in addition to finding an optimal sequence of the assigned locations for each agent to visit. Our prior work in this area has developed several state of the art approximation algorithms for homogeneous~\cite{malik2007approximation, Rathinam_2007_IEEETASE} and heterogeneous \cite{Riddhi_2010, SWAT2020,bae2015primal} agents using Lagrangian relaxation and primal-dual techniques. All of these methods ignore conflicts between the agent's paths. This issue is explicitly handled in this work. 

\vspace{1mm}
\noindent\textbf{Integrated goal allocation and path finding}
%\subsection{Integrated Goal Allocation and Path Finding}
extends MAPF to consider multiple goals~\cite{liu2019task,surynek2020multi,ma2016optimal,henkel2019optimal,honig2018conflict}.
However, most of the prior work either consider the \emph{allocation} aspect of the problem: assigning exactly one goal to each agent from a set of goals, or the \emph{sequencing} aspect of the problem: a set of goals are pre-assigned to each agent and the visiting order is to be determined.
In this work, the multi-goal sequencing in MSMP involves both allocation and sequencing.

%% file: problem_def.tex
% agents and graphs
Let index set $I_N = \{1,2,\dots,N\}$ denote a set of $N$ agents.
All agents move in a workspace represented as a finite graph $G=(V,E)$, where the vertex set $V$ represents the possible locations for agents and the edge set $E =V \times V$ denotes the set of all the possible actions that can move an agent between any two vertices in $V$.
An edge between $u,v\in V$ is denoted as $(u, v)\in E$ and the cost of an edge $e \in E$ is a non-negative real number cost$(e) \in \mathbb{R}^{+}$.

In this article, we use superscript $i \in I_N$ over a variable to represent the specific agent to which the variable belongs (e.g. $v^i\in V$ means a vertex corresponding to agent $i$).
Let $v_o^i\in V$ denote the initial vertex of agent $i$.
There are $N$ \emph{destination} vertices in $G$ denoted by the set $V_d\subseteq V$.
Each agent must terminate its path at any one of these destination vertices. Note here that each destination needs to be assigned to exactly one agent and vice-versa.
In addition, let $V_g = \{t_1,t_2,\dots,t_M\} \subseteq V$ denote the set of $M$ \emph{goal} vertices that must be visited by at least one of the agents in the middle of its path.
For the rest of the article, we will use destinations and goals to represent vertices in $V_d$ and $V_g$ respectively.

Let $\pi^i(v^i_{1}, v^i_{\ell})$ denote a path for agent $i$ that connects vertices $v^i_{1}$ and $v^i_{\ell}$ via a sequence of vertices $(v^i_{1},v^i_{{2}},\dots,v^i_{\ell})$ in the graph $G$. 
Let $g^i(\pi^i(v^i_{1}, v^i_\ell))$ denote the cost associated with the path.
This path cost is the sum of the costs of all the edges present in the path, $i.e.$, $g^i(\pi^i(v^i_{1}, v^i_{\ell})) = \Sigma_{j=1,2,\dots,{\ell-1}} cost(v^i_{{j}}, v^i_{{j+1}})$. 

All agents share a global clock. Each action, either wait or move, requires one unit of time for any agent.
Any two agents $i,j \in I_N$ are claimed to be in conflict if one of the following two cases happens.
The first case is a ``vertex conflict'' where two agents occupy the same vertex at the same time.
The second case is an ``edge conflict'' where two agents travel through the same edge from opposite directions between times $t$ and $t+1$ for some $t$.

The multi-agent simultaneous multi-goal sequencing and path finding (MSMP) problem aims to find a set of conflict-free paths for the agents such that (1) each goal is visited at least once by some agent, (2) the path for each agent starts at its initial vertices and terminates at one of the destinations, and (3) the sum of the cost of the paths is a minimum.

%% file: method.tex
%\graphicspath{{../figures/}}

%\subsection{Notations}\label{sec:math_def}

\subsection{Preliminaries}\label{sec:state_space}

Let $\mathcal{G}=(\mathcal{V},\mathcal{E}) = \underbrace{G \times G \times \dots \times G}_{\text{$N$ times}}$ denote the joint graph which is the cartesian product of $N$ copies of $G$, where each vertex $v \in \mathcal{V}$ represents a joint vertex and $e \in \mathcal{E}$ represents a joint edge that connects a pair of joint vertices.
The joint vertex corresponding to the initial vertices of all agents is $v_o = (v^1_o,v^2_o,\cdots,v^N_o)$. In addition, let $\pi(u,v), u,v \in \mathcal{V}$ represent a joint path, which is a tuple of $N$ individual paths of same length, $i.e.$, $\pi(u,v) = (\pi^1(u^1,v^1), \cdots, \pi^N(u^N,v^N))$. 

Let $a \in \{0,1\}^{M}$ denote a binary vector of length $M$ that indicates the visiting status of goals, $i.e.$, for $m \in \{1,2,\cdots,M\}$, the $m$-th component of $a$, denoted as $a(m)$, is equal to 1 if goal $t_m \in V_g$ is visited by some agent, and $a(m)=0$ otherwise.
To compare two binary vectors, a {\it dominance} relationship is established as follows.
\vspace{.1cm}

\begin{definition}[Binary Dominance]
For any two binary vectors $a_k$ and $a_l$, $a_k$ dominates $a_l$ ($a_k \succeq_b a_l$), if both the following conditions hold:

\begin{itemize}
	\item $a_k(m)\geq a_l(m)$, $\forall m \in \{1,2,\cdots,M\}$. 
	\item $a_k(m)> a_l(m)$, $\exists m \in \{1,2,\cdots,M\}$.
\end{itemize}
\end{definition}
%In addition, two assignment vectors equal each other ($a_k=a_l$) if every corresponding element in both vectors equals.
\vspace{.1cm}

The state space in this work is the cartesian product of the joint graph and the space of binary vectors, $i.e.$, $\mathcal{G} \times \{0,1\}^{M}$.
A state $s=(v,a)$ is a tuple which includes a joint vertex $v\in \mathcal{V}$ occupied by the agents and a binary vector, $a \in \{0,1\}^{M}$. 
For the rest of the work, we use $v(s_k)$ and $a(s_k)$ to denote the joint vertex and binary vector contained in state $s_k$ respectively.
The initial state of all agents is $s_o = (v_o, 0^{M} )$.
The search algorithm conducts A*-like~\cite{astar} search and every state $s_k$ represents a partial solution from $s_o$ to $s_k$, which contains a joint path from $v_o$ to $v(s_k)$. The algorithm terminates when one of the states in the \emph{final} set $S_f:=\{(v_f,1^{M})\}$ is expanded, where $v_f$ is a joint vertex that represents a permutation of the destinations in $V_d$. Additionally, for any state $s_k$ during the search process, as in A*, let $g(s_k)$ denotes the least cost of any joint path found thus far from $s_o$ to $s_k$. To compare two states, we define a dominance rule between them as follows. 
\begin{definition}[State Dominance]
For any two states $s_k$ and $s_l$ with $v(s_k)=v(s_l)$, $s_k$ dominates $s_l$ if either of the following two conditions holds:
\begin{itemize}
	\item $a(s_k) \succeq_b a(s_l)$, $g(s_k) \leq g(s_l)$, or
	\item $a(s_k) = a(s_l)$, $g(s_k) < g(s_l)$.
\end{itemize}
\end{definition}
If $s_k$ does not dominate $s_l$, $s_l$ is then non-dominated by $s_k$.
Let $\alpha(v)=\{s_k | v(s_k)=v\}$ denote a set of states generated by the algorithm at joint vertex $v$.
Initially, $\alpha(v_o)$ contains $s_o$ only, and $\alpha(v)=\emptyset, \forall v \in \mathcal{V}, v\neq v_o$.
When a new state $s_l$ is generated by the algorithm during the search, $s_l$ is inserted into $\alpha(v(s_l))$ only when $s_l$ is non-dominated by any states in $\alpha(v(s_l))$.

A heuristic function is also introduced to guide the search, which maps a state to a non-negative real number ($i.e.$ $h(s) \geq 0$) that underestimates the cost-to-go before finding a solution.
In addition, let $f(s) := g(s)+h(s)$ represent the $f$-value of state $s$, and at any stage of the algorithm, let OPEN denote the open list which contains candidate states to be expanded. As commonly done, we sort the candidate states in OPEN based on their $f$-values.

In addition, let collision set $I_C(s) \subseteq I_N$ represent a set of agents in conflict at state $s$. To detect conflicts, collision function $\Psi:\mathcal{V} \times \mathcal{V} \rightarrow 2^{I_N}$ is introduced to check if there is any vertex or edge conflicts given two neighboring joint vertices. Collision function $\Psi$ returns either an empty set if no conflict is detected, or the set of agents that are in conflict.

\subsection{Key features of the Algorithm}\label{sec.algo_overview}
MS* (Algorithm \ref{alg:ms*}) conducts a M*-like~\cite{wagner2015subdimensional} search in the new state space to find an optimal joint path for the agents.
However, there are {\bf three crucial aspects} where MS* differs from M* that makes it applicable to the MSMP. 
{\bf First}, to construct individual policies and the heuristic value from a state to any (final) state in ${S_f}$, we ignore the conflicts between the agent's paths and solve the resulting mTSP. 
The mTSP in this paper is solved by transforming it to an equivalent TSP and using the LKH algorithm \cite{helsgaun2000effective} (a state of the art TSP solver).
The solution to the transformed TSP can then be readily converted into a path for each agent as explained later in section \ref{sec:heu_and_policy}.
{\bf Second}, when a conflict is found, back-propagating only the subset of agents involved in the conflict may not be sufficient to guarantee an optimal solution for MSMP.
Therefore, we have two variants of the MS* algorithm, one which back-propagates in a similar way as the standard M* and another called MS*-c which back-propagates the entire set of agents.
While the numerical results show that MS* finds optimal solutions for almost all of the instances, our proof of optimality currently only applies to MS*-c.
{\bf Third}, we use the state dominance rules to prune any state from OPEN, which can not be part of an optimal solution.

\begin{algorithm}[htbp]
	\caption{Pseudocode for MS* {\color{blue}(MS*-c)}}\label{alg:ms*}
	\small
	\begin{algorithmic}[1]
		\State{Initialize OPEN with $s_o=(v_o,0^M)$}
		\While{OPEN not empty} %\Comment{Main search loop}
		\State{$s_k=(v_k,a_k) \gets$ OPEN.pop() }
		\State{\textbf{if} $s_k\in S_f$  \textbf{then}}
		\State{\indent $\pi \gets$ \text{Reconstruct($s_k$)}} %\Comment{Reconstruct joint path}
		\State{\indent \textbf{return} $\pi$}
		
%		\State{move $s_k$ from OPEN($v_k$) to CLOSED($v_k$)}
		\State{$S^{ngh} \gets$ \textbf{GetNeighbor}($s_k$) } 
		\ForAll{$s_l=(v_l,a_l) \in S^{ngh}$}
		\State{add $s_k$ to back\_set($s_l$)}
		\State{\textbf{if} $\Psi(v_k,v_l) \neq \emptyset$}
		\State{\indent $I_C(s_l) \gets I_C(s_l) \cup \Psi(v_k,v_l)$}
		\State{\indent {\color{blue}(Notes: for MS*-c, $I_C(s_l) \gets I_N$)}}
		\State{\indent\textbf{BackProp($s_k$, $I_C(s_l)$)}}
% 		\State{\textbf{if} $\Psi(v_k,v_l) \neq \emptyset$}
		\State{\indent\textbf{continue}}
		\State{$f(s_l) \gets g(s_l)$ + $h(s_l)$}
		\State{\textbf{if} $s_l$ is non-dominated by any states in $\alpha(v_l)$ \textbf{then} }
		\State{\indent add $s_l$ to OPEN }
		\State{\indent parent($s_l$) $\gets s_k$}
		\State{\textbf{else} (Notes: $s_l$ is dominated by states in $\alpha(v_l)$)}
%		\State{\indent\textbf{DomBackProp}($s_l$, $I_C(s_l)$)}
		\State{\indent$A \gets \{s_l' | s_l' \in \alpha(v_l), s_l' \succeq s_l \}$.}
		\State{\indent\textbf{for all} $s_l' \in A$ \textbf{do}}
		\State{\indent\indent\textbf{BackProp}($s_k$, $I_C(s_l')$)}
		\State{\indent\indent add $s_k$ to back\_set($s_l'$)}
		\EndFor
		\EndWhile \label{}
		\State{\textbf{return} Failure (no solution)}
	\end{algorithmic}
\end{algorithm}

MS* starts by inserting the initial state $s_o$ into OPEN. In every round of the search, a state $s_k$ with a minimum $f$-value is popped from OPEN. If $s_k\in S_f$, the algorithm outputs a solution by simply tracing the parents of the states iteratively from $s_k$ to $s_o$ and terminates. Otherwise, $s_k$ is further expanded.

To expand $s_k$, all neighboring state in the limited neighbor set (denoted as $S^{ngh}$), of $s_k$ are explored (see Sec.~\ref{sec:expand}).
For each $s_l \in S^{ngh}$, the collision checking function $\Psi(v(s_k),v(s_l))$ is invoked and the resulting collision set $I_C(s_l)$ is back-propagated, which updates the collision set of the parent states recursively.
As noted before, back-propagating $I_C(s)$ may not be enough to guarantee that the algorithm returns an optimal solution.
Thus, in MS*-c, the index set $I_N$, which contains all agents, is back-propagated instead of $I_C(s_l)$, when there is conflict detected at state $s_l$.
Any neighboring state $s_l$ that leads to a conflict is discarded after the collision sets are back-propagated (line 13).
If $s_l$ is not in conflict, $s_l$ is then check for dominance against states in $\alpha(v(s_l))$.
If $s_l$ is not dominated, it is added to OPEN for future expansion.
Otherwise, $s_l$ is discarded and for all states $s_l'$ that dominates $s_l$, $i.e.,\; \{s_l' | s_l' \in \alpha(v_l), s_l' \succeq s_l \}$, $I_C(s_l')$ is backpropagated to $s_k$. By doing so, MS* (and MS*-c) keeps updating the collision set of ancestor states of $s_l$ after $s_l$ is pruned by dominance~\cite{ren2021subdimensional}.
Finally, MS* either finds a solution or terminates when OPEN becomes empty, which indicates there is no solution for the problem. More details of the algorithm are presented in the following subsections.

\subsection{State Expansion}\label{sec:expand}
To expand a state, we leverage the concept of ``limited neighbors'', which was originally introduced in our prior work \cite{wagner2015subdimensional}. The limited neighbors $S^{ngh}$ of a state $s_k$ is a set of neighboring states $s_l$ that can be reached from $s_k$ and is determined by $I_C(s_k)$.
Denote $v_k=v(s_k)$ and for each agent $i$, if $i$ is not part of the collision set ($i \notin I_C(s)$), agent $i$ is only allowed to follow the path $\phi^i(v_k^i,v^i_d)$ which moves it from $v_k^i$ to its assigned destination $v^i_d$ (see Sec.~\ref{sec:heu_and_policy}). If agent $i$ is in the collision set ($i \in I_C(s_k)$), it is allowed to explore all possible neighbors of $v^i_k$ in $G$. The binary vector $a(s_l)$ is updated correspondingly if any un-visited goals get visited in $v^i_l$ by some agent $i \in I_N$, where $v_l=v(s_l)$.

\begin{algorithm}[h]
	\caption{Pseudocode for BackProp}\label{alg.backprop}
	\small
	\begin{algorithmic}[1]
		\State{INPUT: $s$, $I$}% \Comment{$s=(v,a)$}
		\State{\textbf{if} $I \nsubseteq I_C(s)$ \textbf{then}}
		\State{\indent $I_C(s) \gets I \cup I_C(s)$}
		\State{\indent\textbf{if} $s \notin$ OPEN \textbf{then}}
		\State{\indent\indent add $s$ to OPEN}
		\State{\indent\textbf{for all} $s_k \in$ back\_set($s$) \textbf{do}}
		\State{\indent\indent\textbf{BackProp}($s_k$, $I_C(s)$)}
	\end{algorithmic}
\end{algorithm}

$S^{ngh}$ of a state $s$ varies once $I_C(s)$ changes, which dynamically modifies the sub-space embedded in the overall state space that can be reached from $s$. 
Note that $I_C(s)$ is updated recursively when Algorithm \ref{alg.backprop} is called. 
To back-propagate collision sets, a data structure ``back\_set'' is defined at every state. Intuitively, the back\_set at state $s$ contains all parent states from which $s$ is reached.  
When $I_C(s)$ is enlarged, $I_C(s)$ is back-propagated to every state in back\_set($s$).
When back-propagating a collision set $I_C(s)$ to a state $s_k \in$ back\_set($s$), if $I_C(s_k)$ is not a super set of $I_C(s)$, then $I_C(s_k)$ is updated by taking the union of $I_C(s)$ (line 3). 
In addition, $s_k$ is re-opened, which makes MS* expand $s_k$ again with a possibly larger limited neighbor set. 

The aforementioned state expansion strategy is  ``aggressive'' in a sense that, at a state $s$, only agents $i \in I_C(s)$ are allowed to explore all possible moves and agents $i \notin I_C(s)$ are restricted to their individual optimal policies.
However, in MSMP, all agents are ``coupled'' in the space of binary vectors $\{0,1\}^M$ as a goal visited by one agent does not need to be visited by any other agents.
Therefore, when agent $i \in I_C(s)$ changes its next move at state $s$, the optimal policies of other agents $j \notin I_C(s)$ may change.
This requires agents $j \notin I_C(s)$ to consider all possible moves as well in order to guarantee the optimality.
Based on this observation, we propose a ``conservative'' (-c) expansion strategy and name the corresponding MS* algorithm with this strategy as MS*-c.
In MS*-c, when a collision is detected between agents $I_C(s)\subseteq I_N$ at state $s$, instead of back-propagating $I_C(s)$ as MS* does, MS*-c back-propagates $I_N$, the index set that includes all agents.
By doing so, all possible combinations of individual agent moves are considered in the state expansion.
Notice here that MS*-c still differs from a naive A* search in the state space because for states where collision sets are not back-propagated, agents are still restricted to follow their optimal policies.

\subsection{Multi-goal Sequencing, Policies and Heuristics} \label{sec:heu_and_policy}

\noindent\textbf{Multi-goal Sequencing.}
%\subsubsection{Goal sequencing procedure}
Ignoring the possible conflicts, the (multi-)goal sequencing procedure (Fig.~\ref{fig:task_sequencing}) takes a state $s$ as input and outputs a set of \emph{goal sequences} for all agents. Each goal sequence is an ordered list of vertices in $G'$, which starts from current vertex of an agent, as specified in $v(s)$, visits a subset of un-visited goals assigned to the agent in order, and ends at some destination vertex in $V_d$.
This procedure has three steps:
\begin{itemize}
	\item The goal sequencing problem in $G$ is transformed~\cite{oberlin2009transformation} to a single asymmetrical traveling salesman problem (ATSP) in a transformed graph $G'$.
	\item Solve the ATSP in $G'$ and get an optimal tour $\tau_*$ in $G'$.
	\item Partition the optimal tour $\tau_*$ in $G'$ into $N$ goal sequences for the agents in $G$.
\end{itemize}

\begin{figure}
	\centering
	\includegraphics[width=\linewidth,height=9.2cm]{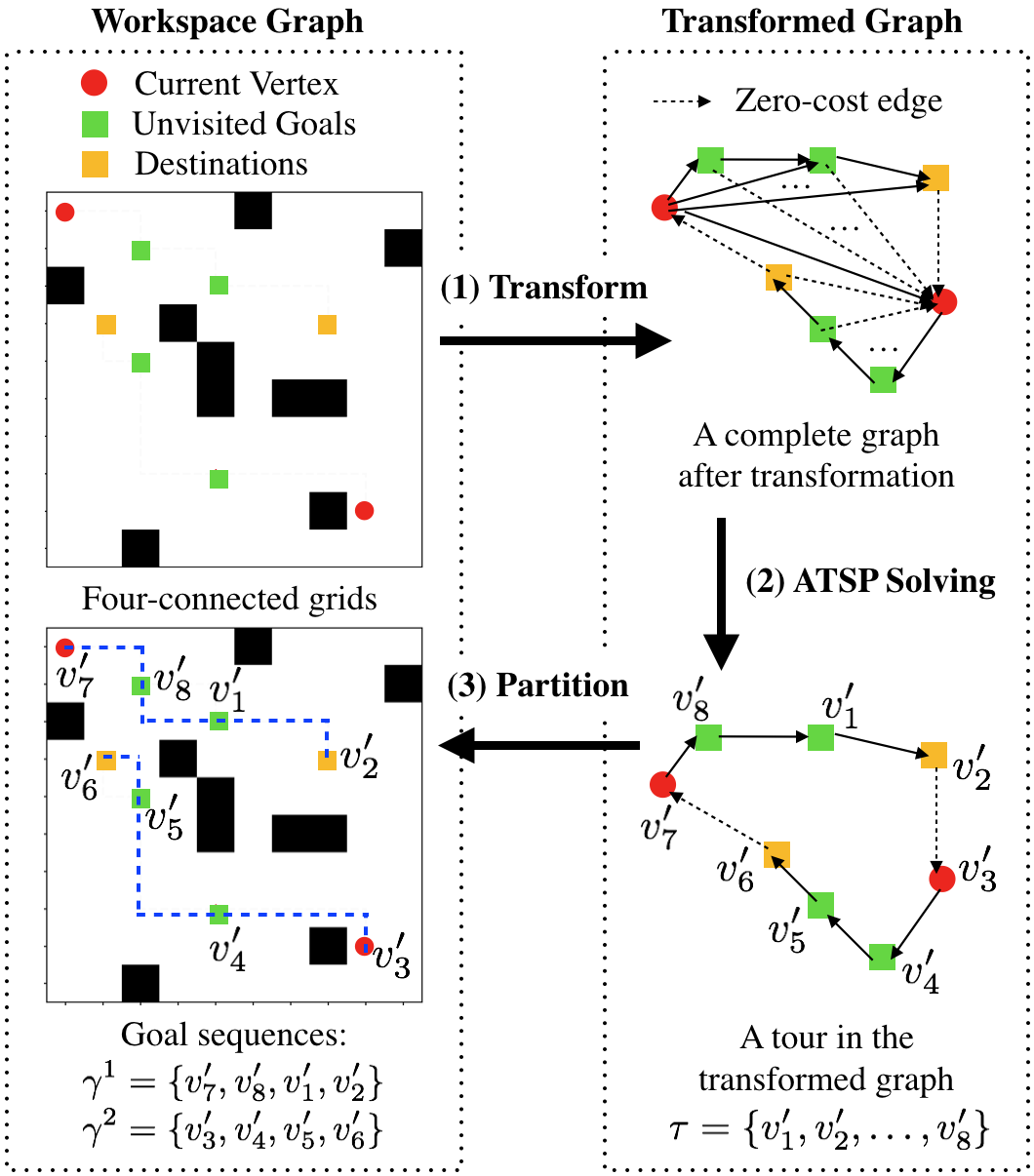}
%	\hspace{-3mm}
	\vspace{-4mm}
	\caption{The workflow of multi-goal sequencing.}
	\label{fig:task_sequencing}
	\vspace{-6mm}
\end{figure}

The transformation step leverages our prior work~\cite{oberlin2009transformation}.
We use symbols with prime (e.g. $G'$) to denote the variables after transformation.
Given a state $s$, let $V_g(s)$ denote the subset of goals that are not visited, $i.e.$, $V_g(s) = \{t_m : a(m) = 0, \forall m = 1,2,\dots,M, a = a(s)\}$.
Let $G'=(V',E')$ represent a transformed graph, where $V'$ is defined as the union of (i) the current vertices of the agents, (ii) the un-visited goals in $V_g(s)$, and (iii) the destinations $V_d$, $i.e.$, $V' = \{v^i | v=v(s), \forall i \in I_N\} \cup V_g(s) \cup V_d$.
Here $G'$ is a complete ($i.e.$ fully-connected) directed graph and the travel cost, $cost'(v'_a,v'_b)$, for any edge joining $v'_a\in V'$ and $v'_b\in V'$ is defined in Table \ref{table:transform_edge_cost}, where $g(\pi_*)$ denotes the cost of a min-cost path in $G$ from $v'_a$ to $v'_b$, which can be computed by any single-agent min-cost path algorithms such as A*~\cite{astar}. %Note here that the edge cost in $G'$ is asymmetrical, $i.e.$ $cost'(v'_a,v'_b) \neq cost'(v'_b,v'_a), \exists (v'_a,v'_b),(v'_b,v'_a) \in E'$.
\begin{table}[h]
	\centering
	\begin{tabular}{ |c|c|c|c| } 
		\hline
		$cost'(v'_a,v'_b)$ & $v'_b \in v$ & $v'_b \in V_g(s)$ & $v'_b \in V_d$ \\
		\hline
		$v'_a \in v$ & $\infty$ & $g(\pi_*)$ & $g(\pi_*)$ \\ 
		\hline
		$v'_a \in V_g(s)$ & $\infty$ & $g(\pi_*)$ & $g(\pi_*)$ \\ 
		\hline
		$v'_a \in V_d$ & 0 & $\infty$ & $\infty$ \\
		\hline
	\end{tabular}
\caption{Edge cost definition in the transformed graph $G'$.}
\vspace{-2mm}
\label{table:transform_edge_cost}
\end{table}

The LKH algorithm~\cite{tinos2018efficient}\footnote{\url{http://akira.ruc.dk/~keld/research/LKH/} }, which is one of the best heuristic-based solver currently available for TSP, is used to find a tour in $G'$.
The LKH, in addition to providing a near-optimal solution, also provides a tight lower bound which can be used as an underestimate for the cost to go from state $s$.
It has been shown in~\cite{oberlin2009transformation} that the near-optimal tour consists of exactly $N$ zero cost edges.
By construction, each of these zero cost edges connects a destination in $V_d$ to a current vertex of an agent.
This allows us to remove all the zero cost edges from the tour and obtain a \emph{goal sequence} $\gamma^i := (u^i_1,u^i_2, \dots, u^i_{\ell^i}), i\in I_N$ for each agent $i=1,\cdots,N$, where $u^i_1=v^i$ is the current vertex of agent $i$, $u^i_{\ell^i}$ is the assigned destination of agent $i$ and all other intermediate vertices $u^i_2,\dots,u^i_{\ell-1}$ are un-visited goals assigned to agent $i$ in order.

\noindent\textbf{Policy Construction and Heuristic Value.}
%\subsubsection{Policy construction}
The goal sequence $\gamma^i$ can be further extended to a path $\pi^i(v^i,u^i_{\ell^i})$ for agent $i$, where any two adjacent vertices in $\gamma^i$ is replaced with a min-cost path between them in $G$.
The path $\pi^i(v^i,u^i_{\ell^i})$ for agent $i$ provides an individual policy $\phi^i$, which maps a vertex to the next vertex along the path and thus direct agent $i$ from $v^i$ to $u^i_{\ell^i}$ if no conflict detected along the path.
%Note that $\phi^i$ is defined over vertices present in path $\pi^i(v^i,u^i_{\ell^i})$.
Consequently, with goal sequencing at state $s$, a joint path $\pi$ is found for all agents, which moves the agents from state $s$ to a final state in $S_f$.
Since any possible conflicts are ignored, the sum of the cost of individual paths in $\pi$ provides an underestimate of the cost-to-go at state $s$ and is used as $h(s)$.

%\vspace{0.2cm}

\noindent{\bf Remark.} The procedure used here is one method to for goal sequencing (or the mTSP).
Several other exact algorithms, such as branch-and-cut~\cite{sundar2015exact} can be used in place of the transformation/LKH method.
An exact algorithm provides an optimal tour for goal sequencing, which in turn allows MS* to provide an optimal solution for MSMP.
Bounded sub-optimal tour leads to inflated heuristic values for states in MS*. As MS* select states based on $f$-values as A* does, inflated heuristic values lead to bounded sub-optimal solutions for the MSMP~\cite{Pearl1982}.
In the ensuing analysis, to outline the main ideas, we assume that the goal sequencing procedure returns an optimal solution. 

%% file: analysis.tex
We provide a sketch proof of the properties of MS*-c assuming the (multi-)goal sequencing returns an optimal tour. 

%\vspace{0.1cm}

\begin{lemma}\label{lem:transform}
	Given a state $s$, let $C(\tau_*)$ denote the cost of the tour $\tau_*$ computed by the goal sequencing procedure and let $\pi_m(v(s),v_f)$ denote a joint path in $G$ with the minimum cost that starts from$s$, visits all un-visited goals and ends at destinations. Then $C(\tau_*) = g(\pi_m(v(s),v_f))$.
\end{lemma}

\ifthenelse{\boolean{shortver}}{% 
}{% 
\begin{proof}
	It is known from \cite{oberlin2009transformation} that the optimal cost of finding $N$ goal sequences for the agents spanning the nodes in $G'$ is equal to the cost of the shortest TSP tour in $G'$. The Lemma follows.
\end{proof}
}

%\vspace{0.1cm}

\begin{lemma}\label{lem:no_sol}
    If there is no feasible solution, MS*-c reports failure in finite time.
\end{lemma}

\ifthenelse{\boolean{shortver}}{% 
}{% 
\begin{proof}
	There is a finite number of states and at each state $s$, the collision set $I_C(s)$ can be modified at most once ($I_C(s)$ is either $\emptyset$ or $I_N$). Thus the sub-space searched by MS*-c will be at most modified for a finite time. Within this sub-space, MS*-c conducts heuristic search (A*-like search) with state dominance pruning. Therefore, if there is no solution, MS*-c terminates in finite time.
\end{proof}
}

%\vspace{0.1cm}

\begin{theorem}\label{thm:optimal}
    MS*-c either computes the optimal solution for MSMP or report failure in finite time.
\end{theorem}
\begin{proof}
    With Lemma \ref{lem:transform}, if no conflict, MS* restricts agents to move along the joint path computed by the goal sequencing procedure and the resulting joint path is guaranteed to be optimal. If there are conflicts at state $s$ along the joint path, with Alg.~\ref{alg.backprop}, all ancestor states of $s$ are updated with $I_C=I_N$ and re-inserted into OPEN for re-expansion with all possible moves considered. Among these moves, MS*-c uses an underestimating heuristic and expands states as in A*, and thus computes an optimal solution.
    With Lemma~\ref{lem:no_sol}, if no solution, MS*-c reports failure in finite time.
\end{proof}

%% file: result.tex
\graphicspath{{figures/}}

\begin{figure*}[htbp]
	\centering
	\vspace{-1mm}
	\includegraphics[width=1.0\linewidth]{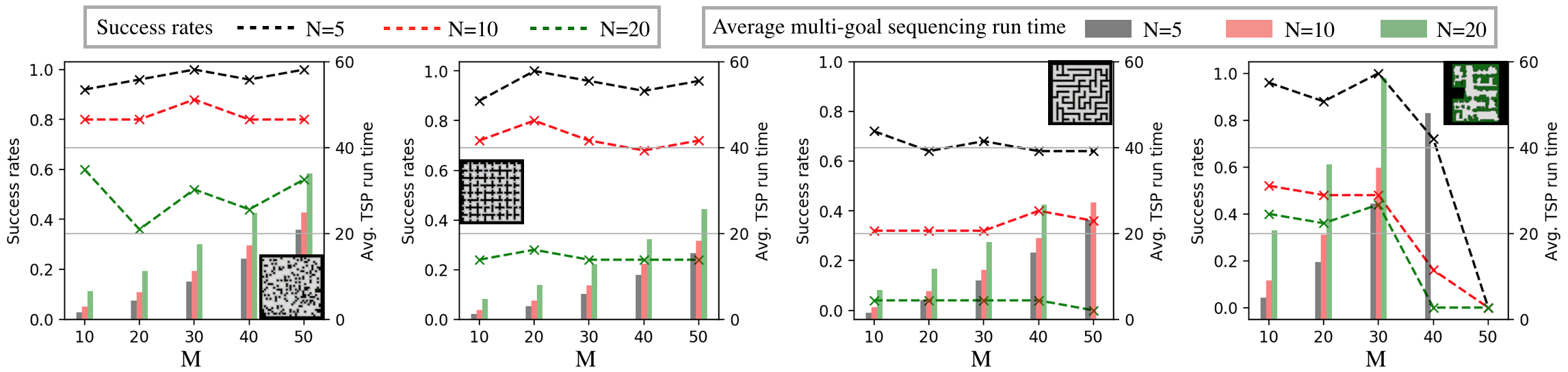}
	\vspace{-6mm}
	\caption{The left vertical axis of each plot represents the success rates of MS* as a function of both the number of agents ($N$) and number of goals ($M$). The right vertical axis of each plot (bar plots) represents the average run time (in seconds) of goal sequencing per call as $M$ varies. The grids used from the left to the right are of type (size): random (32x32), room (32x32), maze (32x32), game map (65x81)}
	\label{fig:all_results}
	\vspace{-5mm}
\end{figure*}

\subsection{Test Settings}\label{sec:experiment_settings}

We implemented both MS* and MS*-c in Python and tested on a computer with a CPU of Intel Core i7 and 16GB RAM.
We selected grids from~\cite{stern2019multi} and make them four-connected.
For each grid, we generated a set of $M$ goals with $M=10,20,\dots,50$ at random location without overlapping with agents' starts or destinations.
We tested with number of agents $N=5,10,20$.
For each test instance, the time bound is \emph{one minute}.
We tested both MS* and MS*-c, and they return the same solution cost for any instance when both of them succeed. In addition, the success rates of MS* is higher than MS*-c within 5\% for all the instances. We thus report only the results of MS* below to make the article concise.

%\begin{figure}
%	\centering
%	\vspace{-3mm}
%	\includegraphics[width=.5\linewidth]{grids_0p2_0.png}
%	\hspace{-6mm}
%	\includegraphics[width=.5\linewidth]{grids_0p2_1.png}
%	\hspace{-6mm}
%	\includegraphics[width=.5\linewidth]{grids_0p2_2.png}
%	\hspace{-6mm}
%	\includegraphics[width=.5\linewidth]{grids_0p2_3.png}
%	\hspace{-6mm}
%	\includegraphics[width=.5\linewidth]{grids_0p2_4.png}
%	\hspace{-3mm}
%	\vspace{-3mm}
%	\caption{Test grids with 20\% random obstacles.}
%	\label{fig:grids}
%\end{figure}

% \begin{figure}
% 	\centering
% 	\hspace{-10mm}
% 	\vspace{-1mm}
% 	\includegraphics[width=\linewidth]{test_case_eg.png}
% 	\vspace{-3mm}
% 	\hspace{-10mm}
% 	\caption{An example of test instances with 20 agents and 50 targets where starts of agents are marked with green circles, targets are marked with blue triangles and goals are marked with red stars.}
% 	\label{fig:test_case}
% \end{figure}

\subsection{Success Rate}\label{sec:succ_rate}
We first measured the success rates of MS* with varying $M$ and $N$.
As shown in Fig.~\ref{fig:all_results}, as $M$ increases, MS* achieves relatively stable success rates except the rightmost game grid, where the initial goal sequencing procedure times out with $M=50$.
When $N$ increases, the chance that agents collide increases when they follow the policies computed by the goal sequencing procedure, which lowers the success rates.
The types of the grid also affects the success rates.
%
%\begin{figure}
%	\centering
%	\hspace{-10mm}
%	\vspace{-1mm}
%	\includegraphics[width=0.9\linewidth]{mtspf_32x32_obst0p2_success_rate_xM.png}
%	\vspace{-3mm}
%	\hspace{-10mm}
%	\caption{Success rates of MS* as a function of the number of agents ($N$) and number of targets ($M$) within one minute computation time. }
%	\label{fig:succ_rate}
%\end{figure}

\subsection{Multi-goal Sequencing Time}\label{sec:tsp_time}
To get a better sense of the computational burden of (multi-)goal sequencing in MS*, we measured the run time of each goal sequencing call and took the average over all the succeeded test instances.
From Fig.~\ref{fig:all_results}, as one can expect, as $N$ or $M$ increases, each call of goal sequencing takes more time.
When $M=50$ and $N=20$, there are $20\text{ agents}+50\text{ goals}+20\text{ destinations}=90$ vertices in the transformed graph $G'$.
In random, room and maze grid, it takes roughly 30 seconds to compute the transformed graph and call LKH to compute a set of goal sequences, which is more than half of the time limit and is very burdensome.
%Speeding up the goal sequencing is an interesting future work.

%
%\begin{figure}
%	\centering
%	\hspace{-10mm}
%	\vspace{-1mm}
%	\includegraphics[width=0.9\linewidth]{mtspf_32x32_obst0p2_tsp_avg_time_xM.png}
%	\vspace{-3mm}
%	\hspace{-10mm}
%	\caption{The computation times required for task sequencing in secs as a function of the number of targets (M). }
%	\label{fig:avg_tsp_time}
%\end{figure}

\begin{comment}
In addition, we also measure the proportion of computation times spent on task sequencing versus the total computation time. From Figure \ref{fig:tsp_percent}, we observe that as $M$ increases, the task sequencing proportion decreases in general. %However, we want to point out an underlying survivor bias in the tests: only ``simiple'' test instances which requires fewer calls of task sequencing are successfully solve by MS*, and therefore counted. For test instances that require several task sequencing computations, MS* times out and the result is not accounted for.

\begin{figure}
	\centering
	\hspace{-10mm}
	\vspace{-1mm}
	\includegraphics[width=0.9\linewidth]{mtspf_32x32_obst0p2_tsp_percentage_xM.png}
	\vspace{-3mm}
	\hspace{-10mm}
	\caption{Average fractions of the total time spent on task sequencing as a function of the number of targets (M). }
	\label{fig:tsp_percent}
\end{figure}
\end{comment}

%% file: conclude.tex
New algorithms, MS* and its variant MS*-c were presented for a multi-agent simultaneous multi-goal sequencing and path finding (MSMP) problem.
The completeness and optimality properties of the algorithms were showed.
Numerical results verifies the performance of both algorithms.
There are several directions for future work.
One can focus on variants of MSMP and MS* can be suitably modified to handle the variants.
One can also focus on fast approximation algorithms for goal sequencing which can improve the performance of the MS*.